\documentclass[11pt,a4paper]{article}
\usepackage[hyperref]{naaclhlt2018}
\usepackage{times}
\usepackage{latexsym}
\usepackage{multirow}
\usepackage{url}
\usepackage{times}  
\usepackage{helvet}  
\usepackage{courier}  
\usepackage{graphicx}  
\usepackage{booktabs}
\usepackage{times}
\usepackage{multicol}
\usepackage{xcolor}
\usepackage{url}
\usepackage{amsmath,amscd,amsbsy,amssymb,latexsym,url,bm}
\renewcommand{\vec}[1]{\mathbf{#1}}  
\newcommand{\bs}{\vec} 
\usepackage{amsmath}
\usepackage{amssymb}
\usepackage{mathrsfs}
\usepackage{mathtools, cuted}
\usepackage{lipsum, color}
\usepackage{amsthm}
\newtheorem{theorem}{Theorem}[section]

\newtheorem{lemma}[theorem]{Lemma}

\newcommand{\argmax}{\arg\max}

\newcommand{\bX}{\mathbf{X}}
\newcommand{\bx}{\mathbf{x}}
\newcommand{\bH}{\mathbf{H}}
\newcommand{\bh}{\mathbf{h}}

\newcommand{\by}{\mathbf{y}}
\newcommand{\bE}{\bs{E}}
\newcommand{\bA}{\bs{A}}
\newcommand{\mL}{\mathcal{L}}

\newcommand{\tabincell}[2]{\begin{tabular}{@{}#1@{}}#2\end{tabular}}
\newlength{\halfpagewidth}
\divide\halfpagewidth by 2

\aclfinalcopy 
\def\aclpaperid{} 

\title{Label-aware Double Transfer Learning for Cross-Specialty\\Medical Named Entity Recognition}

\author{Zhenghui Wang$^\dag$, Yanru Qu$^\dag$, Liheng Chen$^\dag$, Jian Shen$^\dag$, Weinan Zhang$^\dag$\thanks{\ \  Weinan Zhang is the corresponding author.}\\
	\textbf{Shaodian Zhang$^{\dag\ddag}$, Yimei Gao$^\ddag$, Gen Gu$^\ddag$, Ken Chen$^\ddag$, Yong Yu$^\dag$} \\ 
	$^\dag$APEX Data and Knowledge Management Lab, Shanghai Jiao Tong University\\
	$^\ddag$Synyi LLC.  \\
	{\tt \{felixwzh,wnzhang,shaodian\}@apex.sjtu.edu.cn}\\
	{\tt chen.ken@synyi.com}
}

\begin{document}
	\maketitle
	\begin{abstract}
		We study the problem of named entity recognition (NER) from electronic medical records, which is one of the most fundamental and critical problems for medical text mining. Medical records which are written by clinicians from different specialties usually contain quite different terminologies and writing styles. The difference of specialties and the cost of human annotation makes it particularly difficult to train a universal medical NER system. In this paper, we propose a label-aware double transfer learning framework (La-DTL) for cross-specialty NER, so that a medical NER system designed for one specialty could be conveniently applied to another one with minimal annotation efforts. The transferability is guaranteed by two components: (i) we propose label-aware MMD for feature representation transfer, and (ii) we perform parameter transfer with a theoretical upper bound which is also label aware. We conduct extensive experiments on 12 cross-specialty NER tasks. The experimental results demonstrate that La-DTL provides consistent accuracy improvement over strong baselines. Besides, the promising experimental results on non-medical NER scenarios indicate that La-DTL is potential to be seamlessly adapted to a wide range of NER tasks.
	\end{abstract}
	
	\section{Introduction}
	The development of hospital information system and medical informatics drives the leverage of various medical data for a more efficient and intelligent medical care service.
	Among many kinds of medical data, electronic health records (EHRs) are one of the most valuable and informative data as they contain detailed information about the patients and the clinical practices.
	EHRs are essential to many intelligent clinical applications, such as hospital quality control and clinical decision support systems \cite{wu2015named}. 
	Most of EHRs are recorded in an unstructured
	form, i.e., natural language. 
	Hence, extracting structured information from EHRs using natural language processing (NLP), e.g., named entity recognition (NER) and entity linking, plays a fundamental role in medical informatics \cite{zhang2013unsupervised}.
	In this paper, we focus on medical NER from EHRs, which is a fundamental task and is widely studied in the research community \cite{nadeau2007survey,uzuner20112010}. 
	
	In practice, the difficulty of building a universally robust and high-performance medical NER system lies in the variety of medical terminologies and expressions among different departments of specialties and hospitals.
	However, building separate NER systems for so many specialties comes with a prohibitively high cost. 
	The data privacy issue further discourages the sharing of the data across departments or hospitals, making it more difficult to train a canonical NER system to be applied everywhere.
	This raises a natural question: if we have sufficient annotated EHRs data in one \textit{source} specialty, can we distill the knowledge and transfer it to help training models in a related \textit{target} specialty with few annotations? By transferring the knowledge we can achieve higher performance in target specialties with lower annotation cost and bypass the data sharing concerns. This is commonly referred to as \textit{transfer learning} \cite{pan2010survey}.

	Current state-of-the-art transfer learning methods for NER are mainly based on deep neural networks, which perform an end-to-end training to distill sequential dependency patterns in the natural language \cite{ma2016end,lample2016neural}.
	These transfer learning methods
	include (i) feature representation transfer \cite{peng2017multi,kulkarni2016domain}, which normally leverages deep neural networks to learn a close feature mapping between the source and target domains,
	and (ii) parameter transfer \cite{murthy2016sharing,yang2017transfer}, which performs parameter sharing or joint training to get the target-domain model parameters close to those of the source-domain model.
	To the best of our knowledge, there is no previous literature working on transfer learning for NER in the medical domain, or even in a larger scope, i.e., medical natural language processing.

	In this paper, we propose a novel NER transfer learning framework, namely label-aware double transfer learning (La-DTL): 
	(i) We leverage bidirectional long-short term memory (Bi-LSTM) network \cite{graves2005framewise} 
	to automatically learn the text representations, based on which we perform a label-aware feature representation transfer. We propose a variant of maximum mean discrepancy (MMD) \cite{gretton2012kernel}, namely label-aware MMD (La-MMD), to explicitly reduce the domain discrepancy of feature representations of tokens with the same label between two domains.
	(ii) Based on the learned feature representations from Bi-LSTM, two conditional random field (CRF) models are performed for sequence labeling for source and target domain separately, where parameter transfer learning is performed. Specifically, an upper bound of KL divergence between the source and target domain's CRF label distributions is added over the emission and transition matrices across the source and target CRF models to explore the shareable parts of the parameters. 
	Both (i) and (ii) have a label-aware characteristic, which will be discussed later.
	We further argue that label-aware characteristic is crucial for transfer learning in sequence labeling problems, e.g., NER, because only when the corresponding labels are matched, can the ``similar'' contexts (i.e. feature representation) and model parameters be efficiently borrowed to improve the label prediction.

	Extensive experiments are conducted on 12 cross-specialty medical NER tasks with real-world EHRs. The experimental results demonstrate that La-DTL provides consistent accuracy improvement over strong baselines, with overall 2.62\% to 6.70\% absolute F1-score improvement over the state-of-the-art methods. Besides, the promising experimental results on other two non-medical NER scenarios indicate that La-DTL has the potential to be seamlessly adapted to a wide range of NER tasks.

	\section{Related Works}
	\noindent
	{\bf Named Entity Recognition }(NER) is fundamental in information extraction area which aims at automatic detection of named entities (e.g., person, organization, location and geo-political) in free text \cite{marrero2013named}. Many high-level applications such as entity linking \cite{moro2014entity} and knowledge graph construction \cite{hachey2011graph} could be built on top of an NER system. 
	Traditional high-performance approaches include conditional random fields models (CRFs) \cite{lafferty2001conditional}, maximum entropy Markov models (MEMMs) \cite{mccallum2000maximum}  and hidden Markov models (HMMs).  Recently, many neural network-based models have been proposed \cite{collobert2011natural,chiu2015named,ma2016end,lample2016neural}, in which few feature engineering works are needed to train a high-performance NER system. The architecture of those neural network-based models are similar, where different neural networks (LSTMs, CNNs) at different levels (char- and word-level) are applied to learn feature representations, and on top of neural networks, a CRF model is employed to make label predictions.

	\noindent
	\textbf{Transfer Learning} distills knowledge from a source domain to help create a high-performance learner for a target domain. 
	Transfer learning algorithms are mainly categorized into three types, namely instance transfer, feature representation transfer and parameter transfer \cite{pan2010survey}. 
	Instance transfer normally samples or re-weights source-domain samples to match the distribution of the target domain \cite{chen2011co,chu2013selective}.
	Feature representation transfer typically learns a feature mapping which projects source and target domain data simultaneously onto a common feature space following similar distributions
	\cite{zhuang2015supervised,long2015learning,shen2017wasserstein}.
	Parameter transfer normally involves a joint or constrained training for the models on source and target domains, usually introduce connections between source target parameters via sharing \cite{srivastava2013discriminative}, initialization \cite{perlich2014machine}, or inter-model parameter penalty schemes \cite{zhang2016collective}.

	\noindent
	\textbf{Transfer Learning for NER}
	Training a high-performance NER system requires expensive and time-consuming manually annotated data. 
	But sufficient labeled data is critical for the generalization of an NER system, especially for neural network-based models.
	Thus, transfer learning for NER is a practically important problem.
	The first group of methods focuses on sharing model parameters but they differ in the training schemes. 
	\citet{he2017unified} proposed to train the parameter-shared model with source and target data jointly, while the learning rates for sentences from source domain are re-weighted by the similarity with target domain corpus. \citet{yang2017transfer} proposed a family of frameworks which share model parameters in hierarchical recurrent networks to handle cross-application, cross-lingual, and cross-domain transfer in sequence labeling tasks. 
	Differently, \citet{lee2017transfer} first trained the model with source domain data and then fine-tuned the model with little annotated target domain data. 
	
	Domain adaptation method has been well studied in NER scenarios such as using distributed word representations \cite{kulkarni2016domain} and leveraging rule-based annotators \cite{chiticariu2010domain}. Multi-task learning has also been studied to improve performance in multiple NER tasks by transferring meaningful knowledge from other tasks \cite{collobert2011natural,peng2016improving}. To take the advantages of both domain adaptation and multi-task learning, \citet{peng2017multi} proposed a multi-task domain adaptation model.

	\section{Preliminaries}
	This section briefly introduces bidirectional LSTM, conditional random field and maximum mean discrepancy, which are the building blocks of our transfer learning framework.

	\noindent
	{\bf Bidirectional LSTM }
	Recurrent neural networks (RNNs) are widely used in NLP tasks for their great capability to capture contextual information in sequence data.  
	A widely used variant of RNNs is long short-term memory (LSTM) \cite{hochreiter1997long}, which incorporates input and forget gates to capture both long and short term dependencies. 
	Furthermore, it will be beneficial if we process the sequence in not only a forward but also a backward way. 
	Thus, bidirectional LSTM (Bi-LSTM) was employed in many previous works \cite{chiu2015named,ma2016end,lample2016neural} to capture bidirectional information in a sequence. 
	More specifically, for token $\vec{x}_t$ (embedding vector) at timestep $t$ in sequence 
	$\vec{X}=(\vec{x}_{1},\vec{x}_{2},...,\vec{x}_{n})$, the $\theta_b$-parameterized Bi-LSTM recurrently updates  hidden vectors ${\vec{h}_t^{\rightarrow}}=G^f_{\theta_b}(\vec{X},\vec{h}_{t-1}^{\rightarrow})$ and ${\vec{h}_t^{\leftarrow}}=G^b_{\theta_b}(\vec{X},\vec{h}_{t+1}^{\leftarrow})$ produced by a forward LSTM and a backward one, respectively. 
	Then we concatenate ${\vec{h}_t^{\rightarrow}}$ and ${\vec{h}_t^{\leftarrow}}$ to $\vec{h}_t$ as the final hidden vector produced by Bi-LSTM:
	
	\begin{small}
		\begin{equation}
		\vec{h}_t = {\vec{h}_t^{\rightarrow}} \oplus {\vec{h}_t^{\leftarrow}}.\nonumber
		\end{equation}
	\end{small}The representations learned from Bi-LSTM for sequence $\vec{X}$ is thus denoted as $\vec{H} = (\vec{h}_{1},\vec{h}_{2},...,\vec{h}_{n})$.

	\noindent{\bf Conditional Random Field }
	The goal of NER is to detect named entities in a sequence $\vec{X}$ by predicting a sequence of labels $\vec{y} = (y_{1},y_{2},...,y_{n})$. Conditional random field (CRF) is widely used to make joint labeling of the tokens in a sequence \cite{lafferty2001conditional}. 
	
	Recently, \citet{lample2016neural} proposed to build a CRF layer on top of a Bi-LSTM so that the automatically learned feature representation $\vec{H}=(\vec{h}_{1},\vec{h}_{2},...,\vec{h}_{n})$ of the sequence can be directly fed into the CRF for sequence labeling.
	For a sequence of labels $\vec{y}$, given the hidden vector sequence $\vec{H}$, we define its $\theta_c$-parametrized score function $s_{\theta_c}(\vec{H},\vec{y})$ as:
	
	\begin{small}
		\begin{equation}
		s_{\theta_c}(\vec{H},\vec{y}) = \sum_{i=1}^{n}\bs{E}_{i,y_i} + \sum_{i=1}^{n-1}\bs{A}_{y_i,y_{i+1}},\nonumber
		\end{equation}
	\end{small}where $\bs{E}$ is the emission score matrix of size $n \times m$ ($m$ is the number of unique labels), and is computed by 
	$\bs{E} = \vec{H}\bs{W}$ where $\bs{W}$ is the label emission parameter matrix; $\bs{A}$ is the label transition parameter matrix; thus $\theta_c = \{\bs{W}, \bs{A}\}$.
	We then define the conditional probability of label sequence $\vec{y}$ given $\vec{H}$ by a softmax over all possible label sequences in set $\mathcal{Y}(\vec{H})$ as:
	
	\begin{small}
		\begin{align}\label{tab:crf-eq}
		p_{\theta_c}(\vec{y}|\vec{H})=&\exp\{s_{\theta_c}(\vec{H},\vec{y})\}/Z(\vec{H}) \\
		= &\exp\{s_{\theta_c}(\vec{H},\vec{y})\} \Big/ \sum_{\vec{y}'\in\mathcal{Y}(\vec{H})}\exp\{s_{\theta_c}(\vec{H},\vec{y}')\},\nonumber 	
		\end{align}
	\end{small}where $\theta_c$ is omitted for simplification in the following part. 
	The training objective in the CRF layer is to maximize the log-likelihood $\max_{\theta_c}\log p(\vec{y}|\vec{H})$. 
	In the label prediction phase, we give the output label sequence $\vec{y}^*$ with the highest conditional probability $\vec{y}^*=\argmax_{\vec{y}'\in\mathcal{Y}(\vec{H})} p(\vec{y}'|\vec{H})$ by dynamic programming \cite{sutton2012introduction}.

	\noindent
	{\bf Maximum Mean Discrepancy }
	Maximum Mean Discrepancy \cite{gretton2012kernel} is a non-parametric test statistic to measure the distribution discrepancy in terms of the distance between the kernel mean embeddings of two distributions $p$ and $q$. The MMD is defined in particular function spaces that witness the difference in distributions
	
	\begin{small}
		\begin{equation} 
		\text{MMD}(\mathcal{F},p,q) = \sup_{f \in \mathcal{F}}(\mathbb{E}_{x \sim p}[f(x)] - \mathbb{E}_{y \sim q}[f(y)]) .\nonumber
		\end{equation}
	\end{small}By defining the function class $\mathcal{F}$ as the unit ball in a universal Reproducing Kernel Hilbert Space (RKHS), denoted by $\mathcal{H}$, it holds that $\text{MMD}[\mathcal{F},p,q] = 0$ if and only if $p=q$. And then given two sets of samples $X = \{x_1,...,x_m\}$ and $Y = \{y_1,...,y_n\}$ independently and identically distributed (i.i.d.) from $p$ and $q$ on the data space $\mathcal{X}$, the empirical estimate of MMD can be written as the distance between the empirical mean embeddings after mapping to RKHS 
	
	\begin{small}
		\begin{equation} 
		\text{MMD}(X,Y) = \Big\| \frac{1}{m}\sum_{i=1}^{m} \phi(x_i) - \frac{1}{n}\sum_{j=1}^{n} \phi(y_j) \Big\|_\mathcal{H},  \label{eq:MMD} 
		\end{equation} 
	\end{small}where $\phi(\cdot): \mathcal{X} \rightarrow \mathcal{H}$ is the nonlinear feature mapping that induces $\mathcal{H}$.

	\begin{figure}[t]
		\centering
		\includegraphics[width=0.33\textwidth]{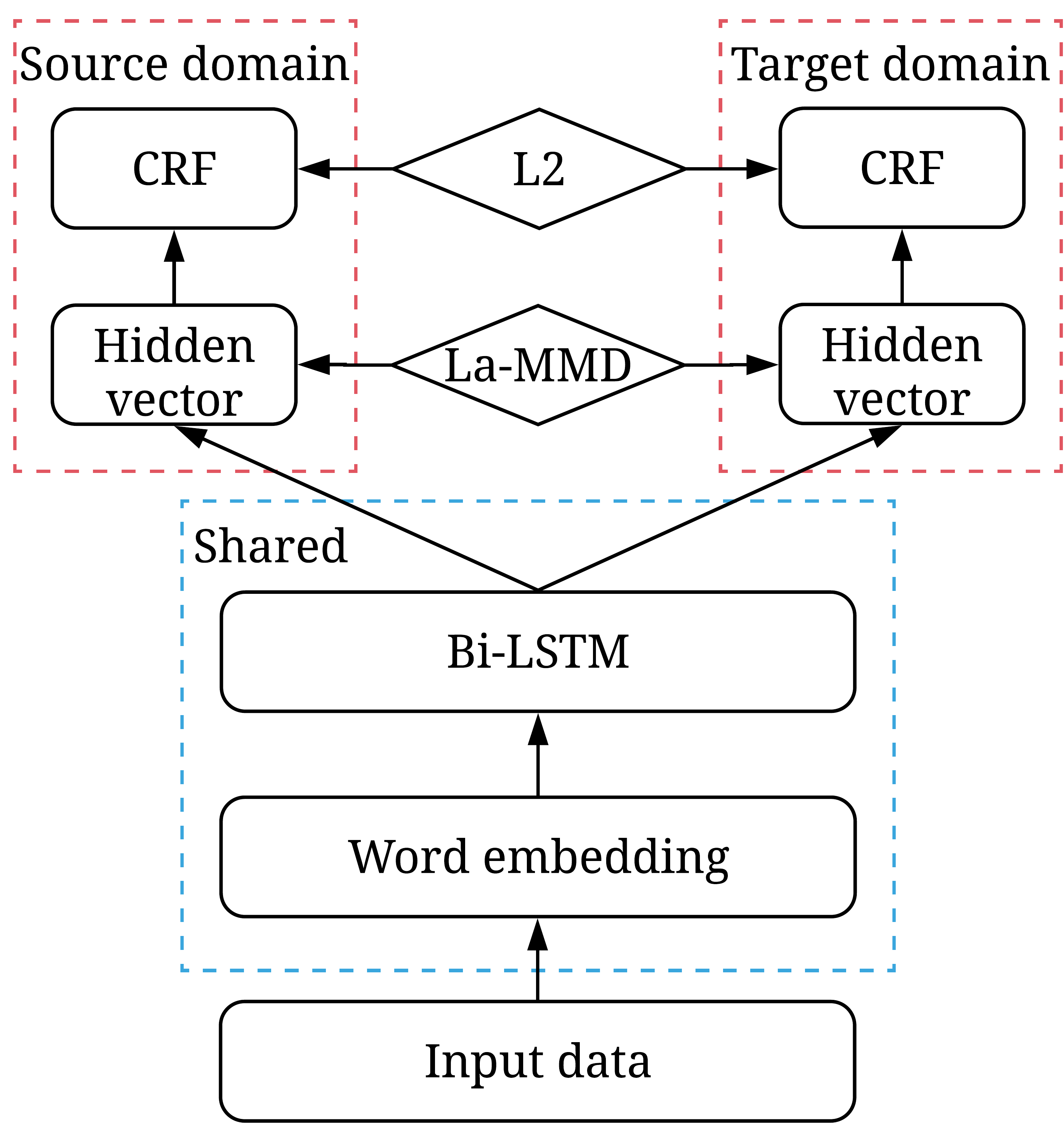}
		\caption{La-DTL framework overview: embedding and Bi-LSTM layers are shared across domains, predictors in red (upper) boxes are task-specific CRFs, with label-aware MMD and L2 constraints to perform feature representation transfer and parameter transfer.
		}
		\label{fig:overview_of_model}
	\end{figure}

	\section{Methodology}
	
	In this section, we present a label-aware double transfer learning (La-DTL) framework and discuss its rationale.

	\subsection{Framework Overview} \label{sec:4-1} 
	Figure~\ref{fig:overview_of_model} gives an overview of La-DTL for NER. 
	From bottom up, each input sentence is converted into a sequence of embedding vectors, which are then fed into a Bi-LSTM to sequentially encode contextual information into fixed-length hidden vectors. The embedding and Bi-LSTM layers are shared among source/target domains. With label-aware maximum mean discrepancy (La-MMD) to reduce the feature representation discrepancy between two domains, the hidden vectors are directly fed into source/target domain specific CRF layers to predict the label sequence. 
	We use domain constrained CRF layers to enhance the target domain performance.

	More formally, let $\mathcal{D}_s={\{(\bX^s_i,\by^s_i)\}}^{N^s}_{i=1}$ be the training set of $N^s$ samples from the source domain
	and $\mathcal{D}_t={\{(\bX^t_i,\by^t_i)\}}^{N^t}_{i=1}$ be the training set of $N^t$ samples from the target domain, with $N^t \ll N^s$. 
	Bi-LSTM encodes a sentence 
	$\bX=(\bx_{1},\bx_{2},...,\bx_{n})$ to hidden vectors $\bH = (\bh_{1},\bh_{2},...,\bh_{n})$.
	We occasionally use $\bH(\bX)$ to denote the corresponding hidden vectors when feeding $\bX$ into the Bi-LSTM.  
	CRF decodes hidden vectors $\bH$ to a label sequence $\hat{\by}=(\hat{y}_{1},\hat{y}_{2},...,\hat{y}_{n})$. Our goal is to improve label prediction accuracy on the target domain $\mathcal{D}_t$ by utilizing the knowledge from the source domain $\mathcal{D}_s$:
	
	\begin{small}
		\begin{align}
		p(\by | \bX) = & p(\by | \bH(\bX)),\nonumber \\
		\log p(\by | \bH) = & \sum_{i=1}^n \bE_{i,y_i} + \sum_{i=1}^{n-1} \bA_{y_i,y_{i+1}} -  \log Z(\bH). \label{eq:pyh}
		\end{align}
	\end{small}Thus training a transferable model $p(\by|\bX)$ requires both $\bH(\bX)$ and $p(\by|\bH)$
	to be transferable.

	We use share word embedding and Bi-LSTM by approaching the feature representation distributions $p(\bh | \mathcal{D}_s)$ and $p(\bh | \mathcal{D}_t)$, i.e., the distributions of Bi-LSTM hidden vectors at each timestep of the sentences from the source and target domains respectively. 
	The rationale behind it lies on the insufficiency of labeled target data.
	Even though LSTM has high capacity, its generalization ability highly relies on viewing ``sufficient'' data. Otherwise, LSTM is very likely to overfit the data. Training on both source and target data, the Bi-LSTM is expected to learn feature representations with high quality. \citet{yosinski2014transferable} provided a justification of this solution that sharing bottom layers is promising for transfer learning in practice.
	
	With the sentences projected onto the same hidden space, the conditional distribution $p(\bh^s | \mathcal{D}_s)$ and $p(\bh^t | \mathcal{D}_t)$, however, may be distant because LSTM hidden vectors contain contextual information which is different across domains. In order to reduce source/target discrepancy, we refine MMD \cite{gretton2012kernel} with label constraints, i.e., label-aware MMD (La-MMD).
	Using La-MMD, the source/target hidden states are pushed to similar distributions to make the feature representation $\bH(\bX)$ transfer feasible.
	
	Based on the hidden vectors from Bi-LSTM, we adopt independent CRF layers for each domain. The rationale lies in the hypotheses that (i) the target domain predictor can better capture target data distribution which could be very unique; (ii) a good predictor trained on the source domain directly could be leveraged to assist the target domain predictor without directly borrowing the source domain training data to bypass the data privacy issue. 
	With respect to the emission and transition score matrices $\sum\bE_{i,y_i}$ and $\sum\bA_{y_i,y_{i+1}}$, we adopt an upper bound between source/target  domains, which helps the target domain predictor to be guided by the source domain predictor. Thus $p(\by|\bH)$ is also transferable.
	
	There are also other transfer methods, including fine-tuning, sharing parameter directly (without constraints) \cite{he2017unified,lee2017transfer,yang2017transfer}, etc. 
	However, simply sharing models may dismiss target specific instances. 
	
	\subsection{Learning Objective}
	The learning objective is to minimize the following loss $\mathcal{L}$ with respect to parameters $\bs{\Theta}=\{\bs{\theta}_b,\bs{\theta}_c\}$:
	
	\begin{small}
		\begin{equation}
		\mL=\mL_{c}+\alpha\ \mL_\text{La-MMD}+\beta\ \mL_{p}+\gamma\ \mL_{r} \label{eq:combined-loss}, \nonumber
		\end{equation}
	\end{small}where $\mL_c$ is the CRF loss, $\mL_\text{La-MMD}$ is the La-MMD loss, $\mL_p$ is the parameter similarity loss on CRF layers, and $\mathcal{L}_r$ is the regularization term, with $\alpha,\beta,\gamma$ as hyperparameters to balance loss terms.
	
	The CRF loss is our ultimate objective predicting the label sequence given the input sentence, i.e., we minimize the negative log-likelihood of training samples from both source/target domains:
	
	{\small
		\begin{align}
		\mathcal{L}_c= 
		- \frac{\varepsilon}{N^s}\sum_{i=1}^{N^s}\log p(\vec{y}_i^s|\bH_i^s)-\frac{1-\varepsilon}{N^t}\sum_{i=1}^{N^t}\log p(\vec{y}_i^t|\bH_i^t), \nonumber 
		\end{align}
	}where $\bH$ are hidden vectors obtained from  Bi-LSTM, $\varepsilon$ is the balance coefficient. The La-MMD loss $\mL_\text{La-MMD}$ and parameter similarity loss $\mL_{p}$  are discussed in Section \ref{sec:4-3} and \ref{sec:4-4}, respectively. The regularization term is to generally control overfitting:
	
	\begin{small}
		\begin{equation}
		\mathcal{L}_r= \|{\theta}_b\|^2_2 + \|\theta_c\|^2_2 .\nonumber
		\end{equation}
	\end{small}
	
	We will provide the model convergence and hyperparameter study in Section \ref{sec:5-1}.

	\begin{figure}[t] 
		\includegraphics[width=0.45\textwidth]{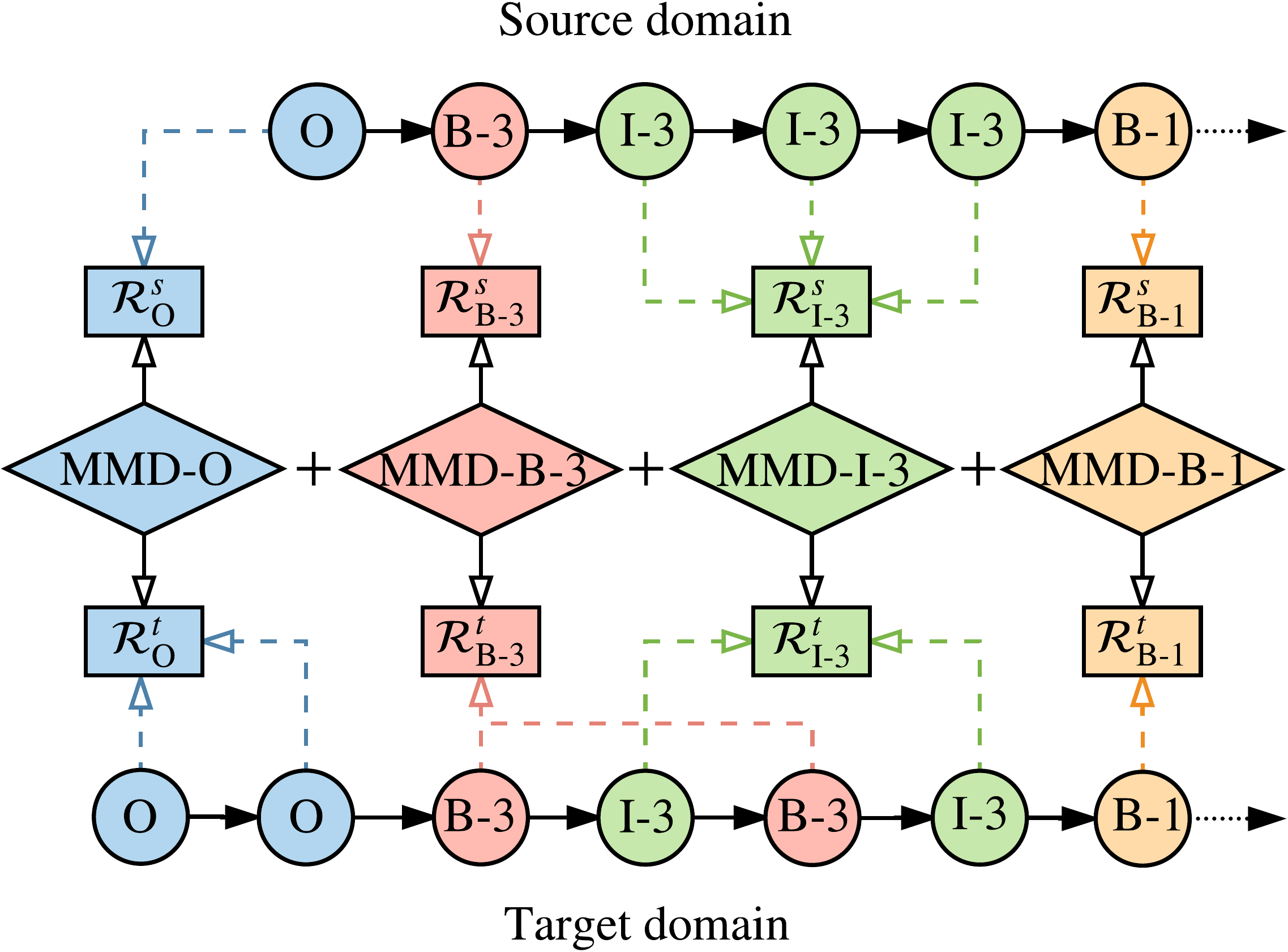} 
		\caption{Illustration for La-MMD. MMD-$y$ is computed between two domains' hidden representations with the same ground truth label $y$. A linear combination is then applied to each label-wise MMD to form La-MMD and the coefficient is set as $\mu_y=1$.} 
		\label{fig:mmd} 
	\end{figure}
	
	\subsection{Bi-LSTM Feature Representation Transfer}\label{sec:4-3}
	To learn transferable feature representations, the maximum mean discrepancy (MMD) which measures the distance between two distributions, has been widely used in domain adaptation scenarios \cite{long2015learning,rozantsev2016beyond}.
	Almost all these works focus on reducing the \emph{marginal} distribution distance between different domain features in an unsupervised manner to make them indistinguishable. However, considering a word is not evenly distributed conditioning on different labels, it may result in that the discriminative property of features from different domains may not be similar, which means that close source and target samples may not have the same label. Different from previous works, we propose label-aware MMD (La-MMD) in Eq.~(\ref{eq:La-MMD}) to explicitly reduce the discrepancy between hidden representations with the same label, i.e., the linear combination of the MMD for each label. For each label class $y\in\mathcal{Y}_v$, where $\mathcal{Y}_v$ is the set of matched labels in two domains, we compute the squared population MMD between the hidden representations of source/target samples with the same label $y$:
	
	\resizebox{0.95\columnwidth}{!}{
		\hspace{-0.6cm}
		\begin{minipage}{1.1\columnwidth}
			
			\begin{small}
				\begin{align}
				\text{MMD}^ 2(\mathcal{R}_{y}^s,\mathcal{R}_{y}^t) 
				=& \frac{1}{({N_y^s})^2}\sum_{i,j=1}^{N_y^s} k(\bh^s_i,\bh^s_j)  + \frac{1}{(N_y^t)^2}\sum_{i,j=1}^{N_y^t} k(\bh^t_i,\bh^t_j) \nonumber \\ 
				&- \frac{2}{N_y^s N_y^t}\sum_{i,j=1}^{N_y^s,N_y^t} k(\bh^s_i,\bh^t_j), \label{eq:single-MMD} 
				\end{align}
			\end{small}
		\end{minipage}}
		where $\mathcal{R}_{y}^s$ and $\mathcal{R}_{y}^t$ are sets of hidden representation $\bh^s$ and $\bh^t$ with corresponding number $N^s_y$ and $N^t_y$. 
		Eq.~(\ref{eq:single-MMD}) can be easily derived by casting Eq.~(\ref{eq:MMD}) into inner product form and applying $\langle \phi(x), \phi(y) \rangle _\mathcal{H} = k(x, y) $ where $k$ is the reproducing kernel function \cite{gretton2012kernel}. 
		For each label class, we compute the MMD loss in a normal manner. After that, we define the La-MMD loss as:
		
		\begin{small}
			\begin{equation}
			\mathcal{L}_\text{La-MMD}=\sum_{y\in\mathcal{Y}_v}\mu_y\cdot\text{MMD}^2(\mathcal{R}_{y}^s,\mathcal{R}_{y}^t),\label{eq:La-MMD} 
			\end{equation}
		\end{small}where $\mu_y$ is the corresponding coefficient. The illustration of La-MMD is shown in Figure~\ref{fig:mmd}.
		
		Once we have applied this La-MMD to our representations learned from Bi-LSTM, the representation distribution of instances with the same label from different domains should be close. Then the standard CRF layer which has a simple linear structure takes these similar representations as input and is likely to give a more transferable label decision for instances with the same label.
		
		\subsection{CRF Parameter Transfer}\label{sec:4-4}
		
		Simply sharing the CRF layer is non-promising when source/target data are diversely distributed. According to probability decomposition in Eq.~(\ref{eq:pyh}), in order to transfer on source/target CRF layers, more specifically, $p(\by|\bH)$,
		we reduce the KL divergence from $p^t(\by|\bH)$ to $p^s(\by|\bH)$. But directly reducing $D_{\text{KL}}(p^s(\by|\bH)||p^t(\by|\bH))$ is intractable, we tend to reduce its upper bound: 
		
		{\small
			\begin{align}
			&D_{\text{KL}}(p^s(\by|\bH)||p^t(\by|\bH))\nonumber\\
			=&\sum_{\by\in\mathcal{Y}(\bH)} p^s(\by|\bH) \log (\frac{p^s(\by|\bH)}{p^t(\by|\bH)})\nonumber\\
			=&-H(p^s(\by|\bH))-\sum_{\by\in\mathcal{Y}(\bH)} p^s(\by|\bH) \log p^t(\by|\bH) \nonumber\\
			\le  
			&c(\Arrowvert\bs{W}^s-\bs{W}^t\Arrowvert^2_2+\Arrowvert\bs{A}^s-\bs{A}^t\Arrowvert^2_2 )^{\frac{1}{2}}, \label{eq:tl-kl-bound}
			\end{align}
		}where $H(\cdot)$ is the entropy of distribution $(\cdot)$ and $c$ is a constant. The detailed proof is provided in Appendix~\ref{sec:a}. Since $c(\Arrowvert\bs{W}^s-\bs{W}^t\Arrowvert^2_2+\Arrowvert\bs{A}^s-\bs{A}^t\Arrowvert^2_2 )$ is the upper bound of $D_{\text{KL}}(p^s(\by|\bH)\|p^t(\by|\bH))$, we conduct CRF parameter transfer 
		by minimizing
		
		\begin{small}
			\begin{equation}
			\mathcal{L}_p={\|\bs{W}^s-\bs{W}^t\|}^2_2+
			{\|\bs{A}^s-\bs{A}^t\|}^2_2. \nonumber
			\end{equation}
		\end{small}It turns out that a similar regularization term is applied in our CRF parameter transfer method and the regularization framework (RF) for domain adaptation \cite{lu2016A}. However, RF is proposed to generalize the feature augmentation method in \cite{Daume2007Frustratingly}, and these two methods are only discussed from a perspective of the parameter. There is no guarantee that two models having similar parameters yields similar output distributions. In this work, we discuss the model behavior in CRF conditions, and we successfully prove that two CRF models having similar parameters (in Euclidean space) yields similar output distributions. In another word, our method guarantees transferability in the model behavior level, while previous works are limited in parameter level.

		The CRF parameter transfer is illustrated in Figure~\ref{fig:crf}, which is also label-aware since the L2 constraint is added over parameters corresponding to the same label in two domains, e.g., $\vec{W}^s_O$ and $\vec{W}^t_O$.
		
		\begin{figure}[t] 
			\centering 
			\includegraphics[width=0.95\columnwidth]{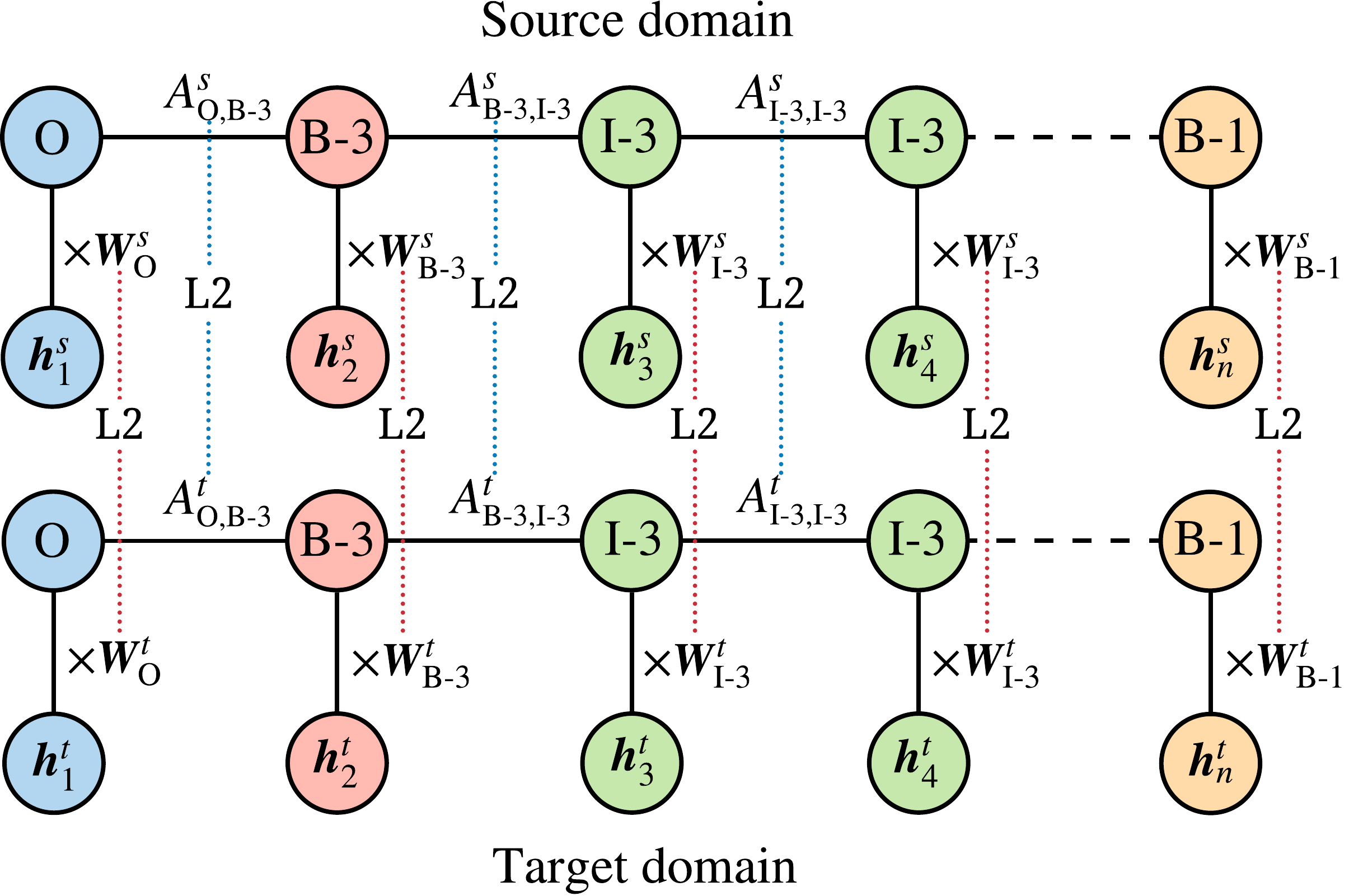} 
			\caption{Illustration for CRF parameter transfer.} 
			\label{fig:crf} 
		\end{figure}
		
		\subsection{Training}
		We train La-DTL in an end-to-end manner with mini-batch AdaGrad \cite{duchi2011adaptive}. One mini-batch contains training samples from both domains, otherwise the computation of $\mathcal{L}_{\text{La-MMD}}$ can not be performed. During training, word (and character) embeddings are fine-tuned to adjust real data distribution. During both training and decoding (testing) of CRF layers, we use dynamic programming to compute the normalizer in Eq.~(\ref{tab:crf-eq}) and infer the label sequence.

		\section{Experiments}
		In this section,  we evaluate La-DTL\footnote{\url{https://github.com/felixwzh/La-DTL}} and other baseline methods on 12 cross-specialty NER problems based on real-world datasets. The experimental results show that La-DTL steadily outperforms other baseline models in all tasks significantly. We also conduct further ablation study and robustness study. 
		We evaluate La-DTL on two more non-medical NER transfer tasks to validate its general efficacy over a wide range of applications.

		\subsection{Cross-Specialty NER} \label{sec:5-1}
		
		{\bf Datasets }
		We collected a Chinese medical NER (\textit{CM-NER}) corpus for our experiments.
		This corpus contains 1600 de-identified EHRs of our affiliated hospital from four different specialties in four departments:
		Cardiology (500), Respiratory (500), Neurology (300) and Gastroenterology (300), and the research had been reviewed and approved by the ethics committee. Named entities are annotated in the BIOES format (Begin, Inside, Outside, End and Single), with 30 types in total. The statistics of \textit{CM-NER} is shown in Table \ref{tab:clinical-corpus}.
		\begin{table}[t]
			\centering
			\small
			\begin{tabular}{l|ccc}
				\toprule
				{Department} &{\# Train} & {\# Dev} &{\# Test}\\
				\midrule
				Cardiology & 3,004 & 601 &601\\
				Respiratory & 3,025 & 605 &606 \\
				Neurology & 932 & 187&187\\
				Gastroenterology & 1,517 & 303 &304 \\
				\midrule
				Sum & 8,478 &1,696 &1,698  \\		
				\bottomrule
			\end{tabular}
			\caption{ Sentence numbers for \textit{CM-NER} corpus. }
			\label{tab:clinical-corpus}
		\end{table}

		\begin{table*}
			\centering 
			
			\tiny
			\begin{tabular} {p{3.6cm}|p{0.4cm}p{0.4cm}p{0.4cm}p{0.4cm}p{0.4cm}p{0.4cm}p{0.4cm}p{0.4cm}p{0.4cm}p{0.4cm}p{0.4cm}p{0.6cm}|p{0.6cm}}
				\toprule
				Method&C$\rightarrow$R&C$\rightarrow$N&C$\rightarrow$G&R$\rightarrow$C&R$\rightarrow$N&R$\rightarrow$G&N$\rightarrow$C&N$\rightarrow$R&N$\rightarrow$G&G$\rightarrow$C&G$\rightarrow$R&G$\rightarrow$N&AVG\\
				\midrule		
				Non-transfer&67.20&54.51&49.01&65.63&54.51&49.01&65.63&67.20&49.01&65.63&67.20&54.51&59.09\\
				Linear projection \cite{peng2017multi}&69.01&67.02&57.40&69.79&65.87&57.71&67.70&68.77&51.33&68.00&69.65&61.12&64.45\\
				
				Domain mask \cite{peng2017multi}&70.76&63.97&58.62&70.18&64.27&58.16&67.93&69.89&56.18&68.87&69.89&63.49&65.18\\
				
				CD-learning \cite{he2017unified}&71.38&64.01&56.72&72.17&64.91&58.14&68.99&71.13&56.27&70.17&71.76&62.06&65.64\\
				Re-training \cite{lee2017transfer}&72.45&70.55&59.58&72.56&68.59&60.94&69.60&70.08&56.58&70.14&71.90&66.01&67.42\\
				Joint-training \cite{yang2017transfer}&69.82&70.49&63.52&71.45&67.03&67.71&70.96&71.43&60.54&69.68&71.55&68.15&68.53\\
				
				\midrule	
				La-MMD&73.08&69.48&59.86&72.53&70.28&60.16&71.31&73.04&57.94&69.80&73.99&67.19&68.22\\
				CRF-L2 &73.34&71.52&60.17&72.43&69.72&67.61&69.76&71.54&59.96&69.75&71.82&67.30&68.74\\
				MMD-CRF-L2&73.05&72.35&60.80&72.65&69.87&66.82&70.25&71.75&58.98&70.48&73.98&67.43&69.03 \\
				La-DTL&\textbf{73.59}$^\dagger$&\textbf{72.91}$^\dagger$&\textbf{64.60}$^\dagger$&\textbf{73.88}$^\dagger$&\textbf{73.01}$^\dagger$&\textbf{70.17}$^\dagger$&\textbf{73.08}$^\dagger$&\textbf{73.11}$^\dagger$&\textbf{62.14}$^\dagger$&\textbf{71.61}$^\dagger$&\textbf{74.21}$^\dagger$&\textbf{71.49}$^\dagger$&\textbf{71.15}\\
				\bottomrule
			\end{tabular}
			
			\caption{Results (F1-score \%) of 12 cross-specialty medical NER tasks. C, R, N, G are short for the department of Cardiology, Respiratory, Neurology, and Gastroenterology, respectively. $\dagger$ indicates La-DTL outperforms the 6 baselines significantly ($p<0.05$).}
			\label{tab:12transfer}
		\end{table*}
		
		\noindent
		{\bf Baselines }
		The following methods are compared. For a fair comparison, we implement La-DTL and baselines with the same base model introduced in \cite{lample2016neural} but with different transfer techniques.
		\begin{itemize}
			\item \textbf{Non-transfer} uses the target domain labeled data only.
			\item \textbf{Domain mask} and \textbf{Linear projection} belong to the same framework proposed by \citet{peng2017multi} but have different implementations at the projection layer, which aims to produce shared feature representations among different domains through a linear transformation.
			\item \textbf{Re-training} is proposed by \citet{lee2017transfer}, where an artificial neural
			networks (ANNs) is first trained on the source domain and then re-trained on the target domain.
			\item \textbf{Joint-training} is a transfer learning method proposed by \citet{yang2017transfer} where different tasks are trained jointly.
			\item \textbf{CD-learning} is a cross-domain learning method proposed by \citet{he2017unified}, where each source domain training example's learning rate is re-weighted. 
		\end{itemize}
		
		\noindent
		{\bf Experimental Settings }
		We use 23,217 unlabeled clinical records to train the word embeddings (word2vec) at 128 dimensions using skip-gram model \cite{mikolov2013distributed}. The hidden state size is set to be 200 for word-level Bi-LSTM. We evaluate La-DTL for cross-specialty NER with \textit{CM-NER} in 12 transfer tasks, results shown in Table \ref{tab:12transfer}.
		For each task, we take the whole source domain training set $\mathcal{D}_s$ and 10\% sentences of the target domain training set $\mathcal{D}_t$ as training data.
		We use the development set in target domain to search hyper-parameters including training epochs. We then take the models to make the prediction in target domain test set and use F1-score as the evaluation metric. Statistical significance has been determined using a randomization version of the paired sample t-test \cite{cohen1995empirical}.

		\noindent
		{\bf Results and Discussion }
		From the results of 12 cross-specialty NER tasks shown in Table \ref{tab:12transfer}, we find that 
		La-DTL outperforms all the strong baselines in all the 12 cross-specialty transfer learning tasks, with 2.62\% to 6.70\%  F1-score lift over state-of-the-art baseline methods.
		Meanwhile, Linear projection and Domain mask \cite{peng2017multi} do not perform as good as other three baselines, which may be because such linear transformation methods are likely to weaken the representations.
		While other three baseline methods all share the whole model between source/target domains but differ in the training schemes and performance.
		
		To better understand the transferability of La-DTL, we also evaluate three variants of La-DTL: La-MMD, CRF-L2, and MMD-CRF-L2. La-MMD and CRF-L2 have the same networks and loss function as La-DTL but with different building blocks: 
		La-MMD has $\beta=0$, while CRF-L2 has $\alpha=0$. In MMD-CRF-L2, we replace La-MMD loss $\mathcal{L}_\text{La-MMD}$ in La-DTL with a vanilla MMD loss:
		\begin{equation}
		\mathcal{L}_\text{MMD}=\text{MMD}^2(\mathcal{R}^s,\mathcal{R}^t), \nonumber
		\end{equation}
		where $\mathcal{R}^s$ and $\mathcal{R}^t$ are sets of hidden representation from source and target domain.
		Results in Table \ref{tab:12transfer} show that: 
		(i) Using La-MMD alone does achieve satisfactory performance since it outperforms the best baseline Joint-training \cite{yang2017transfer} in 7 of 12 tasks. 
		And it has a significant improvement over Domain mask and Linear projection methods \cite{peng2017multi}, which indicates that using La-MMD to reduce the domain discrepancy of feature representations in sequence tagging tasks is promising. 
		(ii) CRF-L2 is also a promising method when transferring between NER tasks, and it improves the La-MMD method significantly when these two methods are combined to form La-DTL. 
		(iii) Label-aware characteristic is important in sequence labeling problems because there is an obvious performance drop when La-MMD is replaced with a vanilla MMD in La-DTL. But MMD-CRF-L2 still has very competitive performance compared to all the baseline methods. This shows positive empirical evidence that transferring knowledge at both Bi-LSTM feature representation level and CRF parameter level for NER tasks is better than transferring knowledge at only one of these two levels, as discussed in Section~\ref{sec:4-1}.

		\begin{figure}[t] 
			\includegraphics[width=0.47\textwidth]{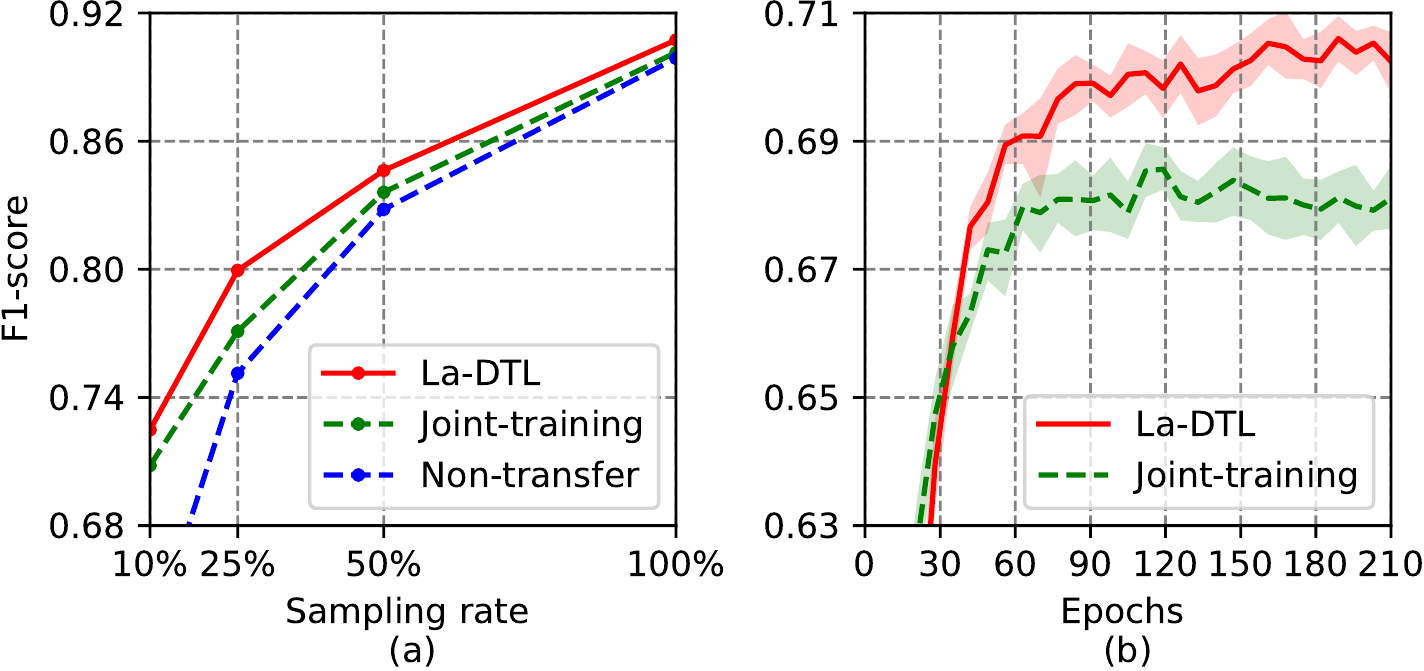} 
			\caption{(a) F1-score of La-DTL, Joint-training and Non-transfer method in C$\rightarrow$R task with different sampling rate. (b) The learning curve of La-DTL and Joint-training in C$\rightarrow$R task.} 
			\label{fig:rate} 
		\end{figure}
		\noindent
		{\bf Robustness to Target Domain Data Sparsity }
		We further study the sparsity problem (target domain) of La-DTL in C$\rightarrow$R task comparing to Joint-training \cite{yang2017transfer} and Non-transfer method. We evaluate La-DTL with different data volume (sampling rate: 10\%, 25\%, 50\%, 100\%) on the target domain training set. Results are shown in Figure~\ref{fig:rate}(a).  We observe that La-DTL outperforms Joint-training and Non-transfer results under all circumstances, and the improvement of La-DTL is more significant when the sampling rate is lower.
		
		To show La-DTL's convergence and significant improvement over Joint-training, we repeat the 10\% sampling rate experiment for 10 times with 10 random seeds.  
		The F1-score on the target domain development set for two methods with a 95\% confidence interval is shown in Figure~\ref{fig:rate}(b) where La-DTL outperforms Joint-training method significantly.
		
		\begin{figure}[t] 
			\includegraphics[width=0.45\textwidth]{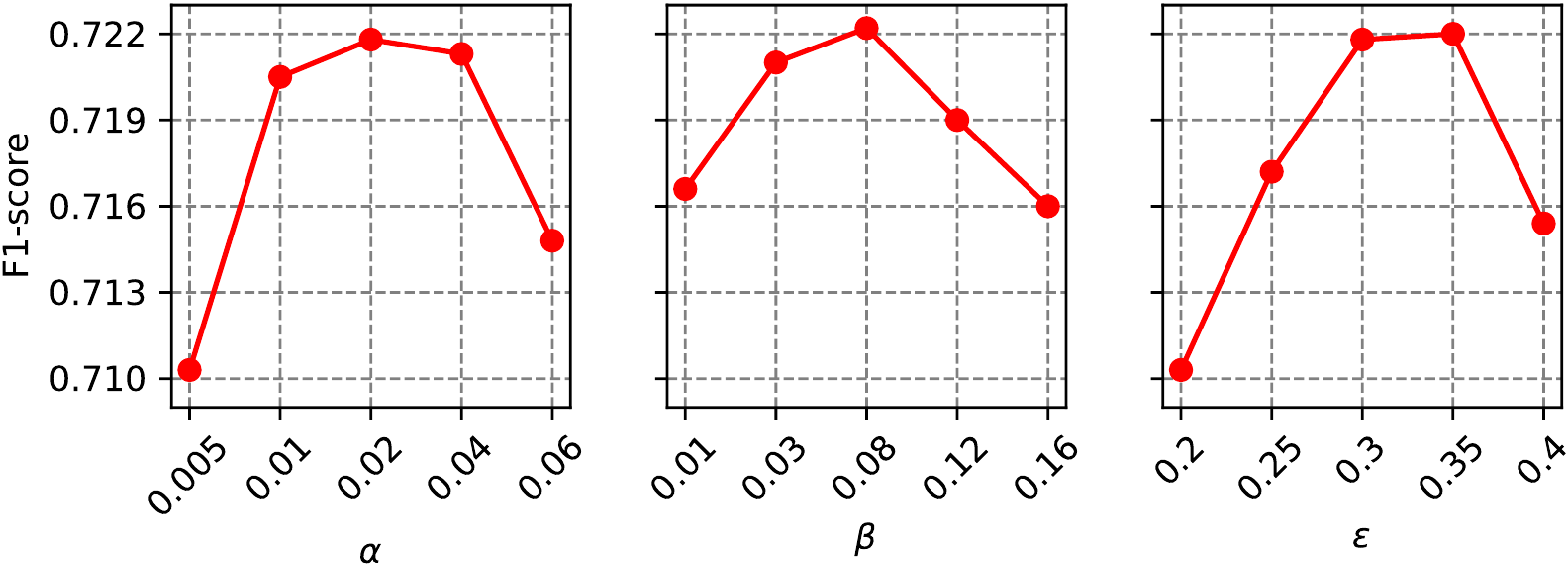} 
			\caption{Hyperparameter study for  $\alpha,\beta,$ and $\varepsilon$.} 
			\label{fig:hyper} 
		\end{figure}
		
		\noindent
		{\bf Hyperparameter Study }
		We study the influence of three key hyperparameters in  La-DTL: $\alpha,\beta,$ and $\varepsilon$ in C$\rightarrow$R task with 10\% target domain sampling rate. 
		We first apply a rough grid search for the three hyperparameters, and the result is ($\alpha=0.02, \beta=0.03, \varepsilon=0.3$). We then fix two hyperparameters and test the third one in a finer granularity. 
		The results in Figure~\ref{fig:hyper} indicate that setting $\alpha \in [0.01,0.04]$ could better leverage La-MMD and further setting $\beta\in [0.03,0.12]$ and $\varepsilon\in [0.3,0.4]$ yields the best empirical performance. This shows that 
		we need to balance the learning objective of the source and target domains for better transferability.

		\subsection{NER Transfer Experiment on Non-medical Corpus}
		To show La-DTL could be applied in a wide range of NER transfer learning scenarios, we make experiments on two non-medical NER tasks. Corpora's details are shown in Table~\ref{tab:sighan}.
		\begin{table}[t]
			\centering 
			\small	
			\begin{tabular}{l|ccc}
				\toprule
				{Corpus} &{\# Train} & {\# Dev} &{\# Test}\\
				\midrule		
				SighanNER & 23,182 &- &4,636  \\
				WeiboNER & 1,350 &270 &270  \\
				CoNLL 2003 & 14,987 &3,466 &3,684  \\
				TwitterNER &1,900  &240 &254  \\
				
				\bottomrule
			\end{tabular}
			\caption{Sentence numbers for non-medical corpora.}
			\label{tab:sighan}
		\end{table}

		\begin{table}
			\centering 
			
			\small
			\begin{tabular} {l|l}
				\toprule		
				Method&F1-score\\
				\midrule
				Non-transfer&54.78\\
				
				Linear projection \cite{peng2017multi}$^*$&56.40\\
				Linear projection \cite{peng2017multi}&56.99\\
				Domain mask \cite{peng2017multi}$^*$&56.80\\
				Domain mask \cite{peng2017multi}&56.32\\
				CD-learning \cite{he2017unified}$^*$&52.05\\
				
				CD-learning \cite{he2017unified}&56.46\\
				Re-training \cite{lee2017transfer}&55.36\\
				Joint-training \cite{yang2017transfer}&56.80\\
				
				\midrule	
				La-DTL&\textbf{57.74}\\
				\bottomrule
			\end{tabular}
			
			\caption{Results (F1-score \%) of \textit{WeiboNER} transfer. \\$*$ indicates the result reported in the corresponding reference.
			}
			\label{tab:weibo}
		\end{table}

		\noindent
		{\bf \textit{WeiboNER} Transfer }
		Following \citet{he2017unified,peng2017multi}, we transfer knowledge from \textit{SighanNER} (MSR corpus of the sixth SIGHAN Workshop on Chinese language processing) to \textit{WeiboNER} (a social media NER corpus) \cite{peng2015named}.
		Results in Table \ref{tab:weibo} show that La-DTL outperforms all the baseline methods in Chinese social media domain.

		\noindent
		{\bf \textit{TwitterNER} Transfer }
		Following \citet{yang2017transfer} we transfer knowledge from CoNLL 2003 English NER \cite{tjong2003introduction} to \textit{TwitterNER} \cite{ritter2011named}. Since the entity types in these two corpora cannot be exactly matched, La-DTL and Joint-training \cite{yang2017transfer} can be applied directly in this case while other baselines can not. Because the CRF parameter transfer of La-DTL is label-aware, and Joint-training simply leverages two independent CRF layers.
		The results are shown in Table \ref{tab:twitter}, where La-DTL again outperforms Joint-training, indicating that La-DTL could be applied seamlessly to transfer learning scenarios with mismatched label sets and languages like English.
		\begin{table}
			\centering 
			
			\small
			\begin{tabular} {l|l}
				\toprule		
				Method&F1-score\\
				\midrule
				Non-transfer &34.65\\
				Joint-training \cite{yang2017transfer}$^*$&43.24\\
				
				\midrule	
				La-DTL&\textbf{45.71}\\
				\bottomrule
			\end{tabular}
			
			\caption{Results (F1-score \%) of \textit{TwitterNER} transfer. \\$*$ indicates the result reported in the corresponding reference.}
			\label{tab:twitter}
		\end{table}

		\section{Conclusions}
		In this paper, we propose La-DTL, a label-aware double transfer learning framework, to conduct both Bi-LSTM feature representation transfer and CRF parameter transfer with label-aware constraints for cross-specialty medical NER tasks. To our best knowledge, this is the first work on transfer learning for medical NER in cross-specialty scenario. Experiments on 12 cross-specialty NER tasks show that La-DTL provides consistent performance improvement over strong baselines. We further perform a set of experiments on different target domain data size, hyperparameter study and other non-medical NER tasks, where La-DTL shows great robustness and wide efficacy. For future work, we plan to jointly perform NER and entity linking for better cross-specialty media structural information extraction.
		
		\section*{Acknowledgments}
		The work done by SJTU is sponsored by Synyi-SJTU Innovation Program, National Natural Science Foundation of China (61632017, 61702327, 61772333) and Shanghai Sailing Program (17YF1428200).
		\bibliography{crf_tl}

\begin{thebibliography}{44}
\expandafter\ifx\csname natexlab\endcsname\relax\def\natexlab#1{#1}\fi

\bibitem[{Chen et~al.(2011)Chen, Weinberger, and Blitzer}]{chen2011co}
Minmin Chen, Kilian~Q Weinberger, and John Blitzer. 2011.
\newblock \href
  {http://papers.nips.cc/paper/4433-co-training-for-domain-adaptation.pdf}
  {Co-training for domain adaptation}.
\newblock In \emph{Advances in Neural Information Processing Systems 24}, pages
  2456--2464. Curran Associates, Inc.

\bibitem[{Chiticariu et~al.(2010)Chiticariu, Krishnamurthy, Li, Reiss, and
  Vaithyanathan}]{chiticariu2010domain}
Laura Chiticariu, Rajasekar Krishnamurthy, Yunyao Li, Frederick Reiss, and
  Shivakumar Vaithyanathan. 2010.
\newblock \href {http://www.aclweb.org/anthology/D10-1098} {Domain adaptation
  of rule-based annotators for named-entity recognition tasks}.
\newblock In \emph{Proceedings of the 2010 Conference on Empirical Methods in
  Natural Language Processing}, pages 1002--1012, Cambridge, MA. Association
  for Computational Linguistics.

\bibitem[{Chiu and Nichols(2016)}]{chiu2015named}
Jason Chiu and Eric Nichols. 2016.
\newblock \href {https://transacl.org/ojs/index.php/tacl/article/view/792}
  {Named entity recognition with bidirectional lstm-cnns}.
\newblock \emph{Transactions of the Association for Computational Linguistics},
  4:357--370.

\bibitem[{Chu et~al.(2013)Chu, De~la Torre, and Cohn}]{chu2013selective}
Wen-Sheng Chu, Fernando De~la Torre, and Jeffery~F Cohn. 2013.
\newblock Selective transfer machine for personalized facial action unit
  detection.
\newblock In \emph{Proceedings of the IEEE Conference on Computer Vision and
  Pattern Recognition}, pages 3515--3522.

\bibitem[{Cohen(1995)}]{cohen1995empirical}
Paul~R Cohen. 1995.
\newblock \emph{Empirical methods for artificial intelligence}, volume 139.
\newblock MIT press Cambridge, MA.

\bibitem[{Collobert et~al.(2011)Collobert, Weston, Bottou, Karlen, Kavukcuoglu,
  and Kuksa}]{collobert2011natural}
Ronan Collobert, Jason Weston, L{\'e}on Bottou, Michael Karlen, Koray
  Kavukcuoglu, and Pavel Kuksa. 2011.
\newblock \href {http://dl.acm.org/citation.cfm?id=1953048.2078186} {Natural
  language processing (almost) from scratch}.
\newblock \emph{J. Mach. Learn. Res.}, 12:2493--2537.

\bibitem[{Daume~III(2007)}]{Daume2007Frustratingly}
Hal Daume~III. 2007.
\newblock \href {http://www.aclweb.org/anthology/P07-1033} {Frustratingly easy
  domain adaptation}.
\newblock In \emph{Proceedings of the 45th Annual Meeting of the Association of
  Computational Linguistics}, pages 256--263. Association for Computational
  Linguistics.

\bibitem[{Duchi et~al.(2011)Duchi, Hazan, and Singer}]{duchi2011adaptive}
John Duchi, Elad Hazan, and Yoram Singer. 2011.
\newblock \href {http://dl.acm.org/citation.cfm?id=1953048.2021068} {Adaptive
  subgradient methods for online learning and stochastic optimization}.
\newblock \emph{J. Mach. Learn. Res.}, 12:2121--2159.

\bibitem[{Graves and Schmidhuber(2005)}]{graves2005framewise}
Alex Graves and J{\"u}rgen Schmidhuber. 2005.
\newblock Framewise phoneme classification with bidirectional lstm and other
  neural network architectures.
\newblock \emph{Neural Networks}, 18(5):602--610.

\bibitem[{Gretton et~al.(2012)Gretton, Borgwardt, Rasch, Sch\"{o}lkopf, and
  Smola}]{gretton2012kernel}
Arthur Gretton, Karsten~M. Borgwardt, Malte~J. Rasch, Bernhard Sch\"{o}lkopf,
  and Alexander Smola. 2012.
\newblock \href {http://dl.acm.org/citation.cfm?id=2188385.2188410} {A kernel
  two-sample test}.
\newblock \emph{J. Mach. Learn. Res.}, 13:723--773.

\bibitem[{Hachey et~al.(2011)Hachey, Radford, and Curran}]{hachey2011graph}
Ben Hachey, Will Radford, and James~R. Curran. 2011.
\newblock \href {http://dl.acm.org/citation.cfm?id=2050963.2050980}
  {Graph-based named entity linking with wikipedia}.
\newblock In \emph{Proceedings of the 12th International Conference on Web
  Information System Engineering}, WISE'11, pages 213--226, Berlin, Heidelberg.
  Springer-Verlag.

\bibitem[{He and Sun(2017)}]{he2017unified}
Hangfeng He and Xu~Sun. 2017.
\newblock A unified model for cross-domain and semi-supervised named entity
  recognition in chinese social media.
\newblock In \emph{AAAI}, pages 3216--3222.

\bibitem[{Hochreiter and Schmidhuber(1997)}]{hochreiter1997long}
Sepp Hochreiter and J{\"u}rgen Schmidhuber. 1997.
\newblock Long short-term memory.
\newblock \emph{Neural computation}, 9(8):1735--1780.

\bibitem[{Kulkarni et~al.(2016)Kulkarni, Mehdad, and
  Chevalier}]{kulkarni2016domain}
Vivek Kulkarni, Yashar Mehdad, and Troy Chevalier. 2016.
\newblock Domain adaptation for named entity recognition in online media with
  word embeddings.
\newblock \emph{arXiv preprint arXiv:1612.00148}.

\bibitem[{Lafferty et~al.(2001)Lafferty, McCallum, and
  Pereira}]{lafferty2001conditional}
John~D. Lafferty, Andrew McCallum, and Fernando C.~N. Pereira. 2001.
\newblock \href {http://dl.acm.org/citation.cfm?id=645530.655813} {Conditional
  random fields: Probabilistic models for segmenting and labeling sequence
  data}.
\newblock In \emph{Proceedings of the Eighteenth International Conference on
  Machine Learning}, ICML '01, pages 282--289, San Francisco, CA, USA. Morgan
  Kaufmann Publishers Inc.

\bibitem[{Lample et~al.(2016)Lample, Ballesteros, Subramanian, Kawakami, and
  Dyer}]{lample2016neural}
Guillaume Lample, Miguel Ballesteros, Sandeep Subramanian, Kazuya Kawakami, and
  Chris Dyer. 2016.
\newblock \href {http://www.aclweb.org/anthology/N16-1030} {Neural
  architectures for named entity recognition}.
\newblock In \emph{Proceedings of the 2016 Conference of the North American
  Chapter of the Association for Computational Linguistics: Human Language
  Technologies}, pages 260--270, San Diego, California. Association for
  Computational Linguistics.

\bibitem[{Lee et~al.(2017)Lee, Dernoncourt, and Szolovits}]{lee2017transfer}
Ji~Young Lee, Franck Dernoncourt, and Peter Szolovits. 2017.
\newblock Transfer learning for named-entity recognition with neural networks.
\newblock \emph{arXiv preprint arXiv:1705.06273}.

\bibitem[{Long et~al.(2015)Long, Cao, Wang, and Jordan}]{long2015learning}
Mingsheng Long, Yue Cao, Jianmin Wang, and Michael Jordan. 2015.
\newblock \href {http://proceedings.mlr.press/v37/long15.html} {Learning
  transferable features with deep adaptation networks}.
\newblock In \emph{Proceedings of the 32nd International Conference on Machine
  Learning}, volume~37 of \emph{Proceedings of Machine Learning Research},
  pages 97--105, Lille, France. PMLR.

\bibitem[{Lu et~al.(2016)Lu, Chieu, and L\"{o}fgren}]{lu2016A}
Wei Lu, Hai~Leong Chieu, and Jonathan L\"{o}fgren. 2016.
\newblock \href {https://aclweb.org/anthology/D16-1095} {A general
  regularization framework for domain adaptation}.
\newblock In \emph{Proceedings of the 2016 Conference on Empirical Methods in
  Natural Language Processing}, pages 950--954, Austin, Texas. Association for
  Computational Linguistics.

\bibitem[{Ma and Hovy(2016)}]{ma2016end}
Xuezhe Ma and Eduard Hovy. 2016.
\newblock \href {http://www.aclweb.org/anthology/P16-1101} {End-to-end sequence
  labeling via bi-directional lstm-cnns-crf}.
\newblock In \emph{Proceedings of the 54th Annual Meeting of the Association
  for Computational Linguistics (Volume 1: Long Papers)}, pages 1064--1074,
  Berlin, Germany. Association for Computational Linguistics.

\bibitem[{Marrero et~al.(2013)Marrero, Urbano, Sánchez-Cuadrado, Morato, and
  Gómez-Berbís}]{marrero2013named}
Mónica Marrero, Julián Urbano, Sonia Sánchez-Cuadrado, Jorge Morato, and
  Juan~Miguel Gómez-Berbís. 2013.
\newblock \href {https://doi.org/https://doi.org/10.1016/j.csi.2012.09.004}
  {Named entity recognition: Fallacies, challenges and opportunities}.
\newblock \emph{Computer Standards \& Interfaces}, 35(5):482 -- 489.

\bibitem[{McCallum et~al.(2000)McCallum, Freitag, and
  Pereira}]{mccallum2000maximum}
Andrew McCallum, Dayne Freitag, and Fernando C.~N. Pereira. 2000.
\newblock \href {http://dl.acm.org/citation.cfm?id=645529.658277} {Maximum
  entropy markov models for information extraction and segmentation}.
\newblock In \emph{Proceedings of the Seventeenth International Conference on
  Machine Learning}, ICML '00, pages 591--598, San Francisco, CA, USA. Morgan
  Kaufmann Publishers Inc.

\bibitem[{Mikolov et~al.(2013)Mikolov, Sutskever, Chen, Corrado, and
  Dean}]{mikolov2013distributed}
Tomas Mikolov, Ilya Sutskever, Kai Chen, Greg~S Corrado, and Jeff Dean. 2013.
\newblock \href
  {http://papers.nips.cc/paper/5021-distributed-representations-of-words-and-phrases-and-their-compositionality.pdf}
  {Distributed representations of words and phrases and their
  compositionality}.
\newblock In \emph{Advances in Neural Information Processing Systems 26}, pages
  3111--3119. Curran Associates, Inc.

\bibitem[{Moro et~al.(2014)Moro, Raganato, and Navigli}]{moro2014entity}
Andrea Moro, Alessandro Raganato, and Roberto Navigli. 2014.
\newblock \href {http://www.aclweb.org/anthology/Q14-1019} {Entity linking
  meets word sense disambiguation: a unified approach}.
\newblock \emph{Transactions of the Association for Computational Linguistics},
  2:231--244.

\bibitem[{Murthy et~al.(2016)Murthy, Khapra, Bhattacharyya
  et~al.}]{murthy2016sharing}
V~Murthy, Mitesh Khapra, Pushpak Bhattacharyya, et~al. 2016.
\newblock Sharing network parameters for crosslingual named entity recognition.
\newblock \emph{arXiv preprint arXiv:1607.00198}.

\bibitem[{Nadeau and Sekine(2007)}]{nadeau2007survey}
David Nadeau and Satoshi Sekine. 2007.
\newblock A survey of named entity recognition and classification.
\newblock \emph{Lingvisticae Investigationes}, 30(1):3--26.

\bibitem[{Pan and Yang(2010)}]{pan2010survey}
Sinno~Jialin Pan and Qiang Yang. 2010.
\newblock \href {https://doi.org/10.1109/TKDE.2009.191} {A survey on transfer
  learning}.
\newblock \emph{IEEE Trans. on Knowl. and Data Eng.}, 22(10):1345--1359.

\bibitem[{Peng and Dredze(2015)}]{peng2015named}
Nanyun Peng and Mark Dredze. 2015.
\newblock \href {http://aclweb.org/anthology/D15-1064} {Named entity
  recognition for chinese social media with jointly trained embeddings}.
\newblock In \emph{Proceedings of the 2015 Conference on Empirical Methods in
  Natural Language Processing}, pages 548--554, Lisbon, Portugal. Association
  for Computational Linguistics.

\bibitem[{Peng and Dredze(2016)}]{peng2016improving}
Nanyun Peng and Mark Dredze. 2016.
\newblock \href {http://anthology.aclweb.org/P16-2025} {Improving named entity
  recognition for chinese social media with word segmentation representation
  learning}.
\newblock In \emph{Proceedings of the 54th Annual Meeting of the Association
  for Computational Linguistics (Volume 2: Short Papers)}, pages 149--155,
  Berlin, Germany. Association for Computational Linguistics.

\bibitem[{Peng and Dredze(2017)}]{peng2017multi}
Nanyun Peng and Mark Dredze. 2017.
\newblock \href {http://www.aclweb.org/anthology/W17-2612} {Multi-task domain
  adaptation for sequence tagging}.
\newblock In \emph{Proceedings of the 2nd Workshop on Representation Learning
  for NLP}, pages 91--100, Vancouver, Canada. Association for Computational
  Linguistics.

\bibitem[{Perlich et~al.(2014)Perlich, Dalessandro, Raeder, Stitelman, and
  Provost}]{perlich2014machine}
Claudia Perlich, Brian Dalessandro, Troy Raeder, Ori Stitelman, and Foster
  Provost. 2014.
\newblock \href {https://doi.org/10.1007/s10994-013-5375-2} {Machine learning
  for targeted display advertising: Transfer learning in action}.
\newblock \emph{Mach. Learn.}, 95(1):103--127.

\bibitem[{Ritter et~al.(2011)Ritter, Clark, Mausam, and
  Etzioni}]{ritter2011named}
Alan Ritter, Sam Clark, Mausam, and Oren Etzioni. 2011.
\newblock \href {http://www.aclweb.org/anthology/D11-1141} {Named entity
  recognition in tweets: An experimental study}.
\newblock In \emph{Proceedings of the 2011 Conference on Empirical Methods in
  Natural Language Processing}, pages 1524--1534, Edinburgh, Scotland, UK.
  Association for Computational Linguistics.

\bibitem[{Rozantsev et~al.(2016)Rozantsev, Salzmann, and
  Fua}]{rozantsev2016beyond}
Artem Rozantsev, Mathieu Salzmann, and Pascal Fua. 2016.
\newblock Beyond sharing weights for deep domain adaptation.
\newblock \emph{arXiv preprint arXiv:1603.06432}.

\bibitem[{Shen et~al.(2017)Shen, Qu, Zhang, and Yu}]{shen2017wasserstein}
Jian Shen, Yanru Qu, Weinan Zhang, and Yong Yu. 2017.
\newblock Wasserstein distance guided representation learning for domain
  adaptation.
\newblock \emph{arXiv preprint arXiv:1707.01217}.

\bibitem[{Srivastava and Salakhutdinov(2013)}]{srivastava2013discriminative}
Nitish Srivastava and Ruslan~R Salakhutdinov. 2013.
\newblock \href
  {http://papers.nips.cc/paper/5029-discriminative-transfer-learning-with-tree-based-priors.pdf}
  {Discriminative transfer learning with tree-based priors}.
\newblock In \emph{Advances in Neural Information Processing Systems 26}, pages
  2094--2102. Curran Associates, Inc.

\bibitem[{Sutton et~al.(2012)Sutton, McCallum et~al.}]{sutton2012introduction}
Charles Sutton, Andrew McCallum, et~al. 2012.
\newblock An introduction to conditional random fields.
\newblock \emph{Foundations and Trends{\textregistered} in Machine Learning},
  4(4):267--373.

\bibitem[{Tjong Kim~Sang and De~Meulder(2003)}]{tjong2003introduction}
Erik~F. Tjong Kim~Sang and Fien De~Meulder. 2003.
\newblock \href {http://www.aclweb.org/anthology/W03-0419.pdf} {Introduction to
  the conll-2003 shared task: Language-independent named entity recognition}.
\newblock In \emph{Proceedings of the Seventh Conference on Natural Language
  Learning at HLT-NAACL 2003}, pages 142--147.

\bibitem[{Uzuner et~al.(2011)Uzuner, South, Shen, and DuVall}]{uzuner20112010}
{\"O}zlem Uzuner, Brett~R South, Shuying Shen, and Scott~L DuVall. 2011.
\newblock 2010 i2b2/va challenge on concepts, assertions, and relations in
  clinical text.
\newblock \emph{Journal of the American Medical Informatics Association},
  18(5):552--556.

\bibitem[{Wu et~al.(2015)Wu, Jiang, Lei, and Xu}]{wu2015named}
Yonghui Wu, Min Jiang, Jianbo Lei, and Hua Xu. 2015.
\newblock Named entity recognition in chinese clinical text using deep neural
  network.
\newblock \emph{Studies in health technology and informatics}, 216:624.

\bibitem[{Yang et~al.(2017)Yang, Salakhutdinov, and Cohen}]{yang2017transfer}
Zhilin Yang, Ruslan Salakhutdinov, and William~W Cohen. 2017.
\newblock Transfer learning for sequence tagging with hierarchical recurrent
  networks.
\newblock In \emph{ICLR}.

\bibitem[{Yosinski et~al.(2014)Yosinski, Clune, Bengio, and
  Lipson}]{yosinski2014transferable}
Jason Yosinski, Jeff Clune, Yoshua Bengio, and Hod Lipson. 2014.
\newblock \href {http://dl.acm.org/citation.cfm?id=2969033.2969197} {How
  transferable are features in deep neural networks?}
\newblock In \emph{Proceedings of the 27th International Conference on Neural
  Information Processing Systems - Volume 2}, NIPS'14, pages 3320--3328,
  Cambridge, MA, USA. MIT Press.

\bibitem[{Zhang and Elhadad(2013)}]{zhang2013unsupervised}
Shaodian Zhang and No{\'e}mie Elhadad. 2013.
\newblock Unsupervised biomedical named entity recognition: Experiments with
  clinical and biological texts.
\newblock \emph{Journal of biomedical informatics}, 46(6):1088--1098.

\bibitem[{Zhang et~al.(2016)Zhang, Paquet, and Hofmann}]{zhang2016collective}
Weinan Zhang, Ulrich Paquet, and Katja Hofmann. 2016.
\newblock Collective noise contrastive estimation for policy transfer learning.
\newblock In \emph{AAAI}, pages 1408--1414.

\bibitem[{Zhuang et~al.(2015)Zhuang, Cheng, Luo, Pan, and
  He}]{zhuang2015supervised}
Fuzhen Zhuang, Xiaohu Cheng, Ping Luo, Sinno~Jialin Pan, and Qing He. 2015.
\newblock \href {http://dl.acm.org/citation.cfm?id=2832747.2832823} {Supervised
  representation learning: Transfer learning with deep autoencoders}.
\newblock In \emph{Proceedings of the 24th International Conference on
  Artificial Intelligence}, IJCAI'15, pages 4119--4125. AAAI Press.

\end{thebibliography}
		
		\bibliographystyle{acl_natbib}
		
		\newpage
		\appendix
		\section{Appendix}
		\subsection{Detailed Proof} \label{sec:a}

		Recall the bound as in Eq.~(\ref{eq:tl-kl-bound}):
		~\\
		~\\
		
		\begin{strip}

			\begin{lemma}\label{lemma:score}  
				$c_1(\Arrowvert\bs{W}^s-\bs{W}^t\Arrowvert^2_2+\Arrowvert\bs{A}^s-\bs{A}^t\Arrowvert^2_2 )$ is the upper bound of $(s^s(\bH,\by)-s^t(\bH,\by))^2$.
			\end{lemma}

			\begin{proof}[Proof of Lemma \ref{lemma:score}]
				$\otimes$ refers to convolutional product, $\vec{H}^W,\vec{H}^A$ are mask matrices corresponding to the given hidden vectors $\vec{H}$, and $c_1$ is a constant. We have:
				\begin{align}\label{lemma:score-l2}
				& (s^s(\bH,\by)-s^t(\bH,\by))^2\nonumber\\
				= & (\sum_{i=1}^{n}\bE_{i, y_i}^s + \sum_{i=1}^{n-1}\bA_{y_i, y_{i+1}}^s-\sum_{i=1}^{n}\bE_{i, y_i}^t - \sum_{i=1}^{n-1}\bA_{y_i, y_{i+1}}^t)^2 \nonumber\\
				= & (\bs{W}^s\otimes\vec{H}^W+\bs{A}^s\otimes\vec{H}^A-\bs{W}^t\otimes\vec{H}^W-\bs{A}^t\otimes\vec{H}^A)^2 \nonumber\\
				= & ((\bs{W}^s-\bs{W}^t)\otimes\vec{H}^W+(\bs{A}^s-\bs{A}^t)\otimes\vec{H}^A)^2 \nonumber\\
				\le &  2((\bs{W}^s-\bs{W}^t)\otimes\vec{H}^W)^2+2((\bs{A}^s-\bs{A}^t)\otimes\vec{H}^A)^2 \nonumber\\
				= & 2(\sum_{i,j}(\bs{W}^s-\bs{W}^t)_{i,j}\cdot\vec{H}^W_{i,j})^2+ \nonumber  2(\sum_{p,q}(\bs{A}^s-\bs{A}^t)_{p,q}\cdot\vec{H}^A_{p,q})^2 \nonumber\\
				\le & 2(\sum_{i,j}(\bs{W}^s-\bs{W}^t)^2_{i,j}\cdot\sum_{i,j}(\vec{H}^W_{i,j})^2)+ 2(\sum_{p,q}(\bs{A}^s-\bs{A}^t)^2_{p,q}\cdot\sum_{p,q}(\vec{H}^A_{p,q})^2)\nonumber\\
				= & 2(\Arrowvert\bs{W}^s-\bs{W}^t\Arrowvert^2_2\cdot\Arrowvert\vec{H}^W\Arrowvert^2_2)+2(\Arrowvert\bs{A}^s-\bs{A}^t\Arrowvert^2_2\cdot\Arrowvert\vec{H}^A\Arrowvert^2_2)\nonumber\\
				\le & c_1(\Arrowvert\bs{W}^s-\bs{W}^t\Arrowvert^2_2+\Arrowvert\bs{A}^s-\bs{A}^t\Arrowvert^2_2 ).\nonumber
				\end{align}
			\end{proof}
			
			\begin{lemma}\label{lemma:kl}  
				$c(\Arrowvert\bs{W}^s-\bs{W}^t\Arrowvert^2_2+\Arrowvert\bs{A}^s-\bs{A}^t\Arrowvert^2_2 )^{\frac{1}{2}}$ is the upper bound of $D_{\text{KL}}(p^s(\by|\bH)||p^t(\by|\bH))$. 
			\end{lemma}  
			
			\begin{proof}[Proof of Lemma \ref{lemma:kl}]
				
				With Lemma.~(\ref{lemma:score}), we set $\varepsilon=(c_1 (\Arrowvert\bs{W}^s-\bs{W}^t\Arrowvert^2_2+\Arrowvert\bs{A}^s-\bs{A}^t\Arrowvert^2_2 ))^\frac{1}{2}\ge0$ and $c=2c_1^{\frac{1}{2}}$, and we have:
				\begin{equation}\label{eq:6}
				s^s(\bH,\by)-\varepsilon\le s^t(\bH,\by)\le s^s(\bH,\by)
				+\varepsilon ,
				\end{equation}
				\begin{equation}\label{eq:7}
				\log \{\sum_{\by'\in\mathcal{Y}(\bH)} \exp[s^s(\bH,\by')]\} -\varepsilon\le 
				\log \{\sum_{\by'\in\mathcal{Y}(\bH)} \exp[s^t(\bH,\by')]\} \le \log \{\sum_{\by'\in\mathcal{Y}(\bH)} \exp[s^s(\bH,\by')]\} +\varepsilon .
				\end{equation}
				
				With Eq.~(\ref{eq:6}) and Eq.~(\ref{eq:7}), we can derive
				\begin{align}
				&- \sum_{\by\in\mathcal{Y}(\bH)} p^s(\by|\bH) \log p^t(\by|\bH)\nonumber\\
				=&-\sum_{\by\in\mathcal{Y}(\bH)} p^s(\by|\bH) \log \frac{\exp[s^t(\bH,\by)]}{\sum_{\by'\in\mathcal{Y}(\bH)} \exp[s^t(\bH,\by')]}\nonumber \\
				= &-\sum_{\by\in\mathcal{Y}(\bH)} p^s(\by|\bH) \big\{ {s^t(\bH,\by)}- {\log \{\sum_{\by'\in\mathcal{Y}(\bH)} \exp[s^t(\bH,\by')]\}}\big\} \nonumber \\
				\le &-\sum_{\by\in\mathcal{Y}(\bH)} p^s(\by|\bH) \big\{ {s^s(\bH,\by)-\varepsilon}- {\log \{\sum_{\by'\in\mathcal{Y}(\bH)} \exp[s^s(\bH,\by')]\}}  -\varepsilon \big\} \nonumber \\
				= &-\sum_{\by\in\mathcal{Y}(\bH)} p^s(\by|\bH) \big\{ \log \frac{\exp[s^s(\bH,\by)]}{\sum_{\by'\in\mathcal{Y}(\bH)} \exp[s^s(\bH,\by')]}{-2\varepsilon}\big\} \nonumber \\
				= &-\sum_{\by\in\mathcal{Y}(\bH)} p^s(\by|\bH) \big\{ \log p^s(\by|\bH) {-2\varepsilon}\big\} \nonumber \\
				= & H(p^s(\by|\bH))+2\varepsilon. \nonumber \\
				\nonumber
				\end{align}
				
				Finally, we have
				\begin{align}
				&D_{\text{KL}}(p^s(\by|\bH)||p^t(\by|\bH))\nonumber\\
				=&\sum_{\by\in\mathcal{Y}(\bH)} p^s(\by|\bH) \log (\frac{p^s(\by|\bH)}{p^t(\by|\bH)})\nonumber\\
				=&-H(p^s(\by|\bH))-\sum_{\by\in\mathcal{Y}(\bH)} p^s(\by|\bH) \log p^t(\by|\bH) \nonumber\\
				\le&-H(p^s(\by|\bH))+H(p^s(\by|\bH))+2\varepsilon\nonumber\\
				=&c(\Arrowvert\bs{W}^s-\bs{W}^t\Arrowvert^2_2+\Arrowvert\bs{A}^s-\bs{A}^t\Arrowvert^2_2 )^{\frac{1}{2}}.
				\nonumber
				\nonumber
				\end{align}
				
			\end{proof}
			
		\end{strip}

		\subsection{Case Analysis}
		
		\begin{table*}[!htb]
			\small
			
			\begin{tabular}{c|c|c|c|c|c|c}
				\toprule
				Disease & \tabincell{c}{Transfer\\ Task } & \tabincell{c}{\# disease term in\\ source  domain\\  training set} & \tabincell{c}{\# disease term in\\ target domain \\  training set }& \tabincell{c}{\# disease term in\\ target domain\\   test set}& \tabincell{c}{\# accurate  \\labeling  \\without transfer} & \tabincell{c}{\# accurate  \\ labeling \\with transfer}\\
				\midrule
				\multirow{3}{*}{\tabincell{c}{rheumatic \\heart disease}} & C$\rightarrow$G&17 & \multirow{3}{*}{0}& \multirow{3}{*}{3}& \multirow{3}{*}{0}&3 \\
				&N$\rightarrow$G &0 && & &0 \\
				& R$\rightarrow$G&16 && & &3 \\
				\midrule
				\multirow{3}{*}{\tabincell{c}{pulmonary\\ heart disease}} & C$\rightarrow$G&4 & \multirow{3}{*}{0}& \multirow{3}{*}{2}& \multirow{3}{*}{0}&2 \\
				& N$\rightarrow$G&0 && & &0 \\
				& R$\rightarrow$G&24 && & &2 \\
				\midrule
				\multirow{3}{*}{\tabincell{c}{coronary \\atherosclerotic\\heart disease}} & G$\rightarrow$N& 5& \multirow{3}{*}{0}& \multirow{3}{*}{15}& \multirow{3}{*}{10}&3 \\
				& C$\rightarrow$N& 136&& & &15 \\
				& R$\rightarrow$N& 23&& & &11 \\
				\bottomrule
			\end{tabular}
			\caption{Case analysis for cross-specialty medical NER tasks. C, R, N, G are short for department of Cardiology, Respiratory, Neurology, and Gastroenterology, respectively. }
			\label{tab:case}
		\end{table*}

		In clinical practice, patients with specific diseases would be assigned to different departments, and specialist doctors in their department may pay more attention to the specific disease. When writing a medical chart, these specific diseases and related clinical findings would have a more detailed description. Therefore, some medical terms would have enriched meanings in different departments accordingly. For example, patients with rheumatic heart disease are often treated in the department of Cardiology. The term, ``rheumatic'', a modifier, describes and limits the type of ``heart disease''. In English, ``rheumatic'' is an adjective modifying ``heart disease''. However, in Chinese, ``rheumatic heart disease'' can be regarded as two diseases, ``rheumatism'' and ``heart disease''. In the department of Cardiology, ``rheumatic heart disease'' is usually mentioned as a single term. While in other departments, ``rheumatism'' and ``heart disease'' are mostly two independent named entities in annotated datasets. As such, it is difficult to train an NER model to capture the relationship between ``rheumatism'' and ``heart disease'', and band them as a whole.  In the training set of our study, the diagnostic term ``rheumatic heart disease'' (including synonym) is mentioned for 17 times in Dept. Cardiology, 16 times in Dept. Respiratory, none in Dept. Neurology and 3 times in Dept. Gastroenterology. We use the data from the first 3 departments as source domain training set respectively, and the data from Dept. Gastroenterology as the target domain training set. We test our models on the test set from Dept. Gastroenterology, where  ``rheumatic heart disease'' is mentioned 3 times, and compare the results across models with/without transfer learning. As expected, models with source training data from Dept. Cardiovascular and Respiration correctly predict all these entities, but the model using source data from Dept. Neurology fails and so does a model without transfer learning.
		
		Patients with pulmonary heart disease were often referred to Dept. Respiratory and Dept. Cardiology. In our training set, ``pulmonary heart disease'' (including synonym) is labeled for 24 times in Dept. Respiratory and 4 times in Dept. Cardiology. In English, ``pulmonary'' modified ``heart disease''. In Chinese, ``pulmonary heart disease'' contains body structure ``lung'' and disease name ``heart disease''. The model trained with the source set from both from department of respiratory and cardiology could correctly recognize the relation between lung and heart disease and predict the entity in the test set from Dept. Gastroenterology.
		
		Similarly, ``coronary atherosclerotic heart disease'' contains two disease names, ``coronary atherosclerosis'' and ``heart disease''. Training model using source set from a department where the terms are enriched could improve the performance of recognizing the whole entity.
		
		\subsection{Medical Experiments Details}
		The 30 entity types for medical domain are: Symptom, Disease, Examination, Treatment, Laboratory index, Products, Body structure, Frequency, Negative word, Value, Trend, Modification, Temporal word, Noun of locality, Degree modifier, Probability, Object, Organism, Location, Person, Pronoun, Privacy information, Accident, Action, Header, Instrument and material, Non-physiological structure, Dosage, Scale, and Preposition.
		
		\subsection{Non-medical Experiments Details}
		
		\subsubsection*{\textit{WeiboNER} Transfer}
		Both \textit{SighanNER} and \textit{WeiboNER} are annotated in the BIO format (Begin, Inside and Outside), but there is one more entity type (geo-political) in \textit{WeiboNER}. For a fair comparison, we follow \citet{peng2017multi,he2017unified} to merge geo-political entities and locations in \textit{WeiboNER}, to match different labeling schemes between \textit{WeiboNER} and \textit{SighanNER}. We use the inconsistencies fixed second version of \textit{WeiboNER} data and word embeddings provided by \textit{WeiboNER}'s developers \cite{peng2015named}\footnote{\scriptsize \url{https://github.com/hltcoe/golden-horse}} in this experiment. 
		
		\subsubsection*{\textit{TwitterNER} Transfer}
		
		\begin{figure}[t]
			\centering
			\includegraphics[width=0.3\textwidth]{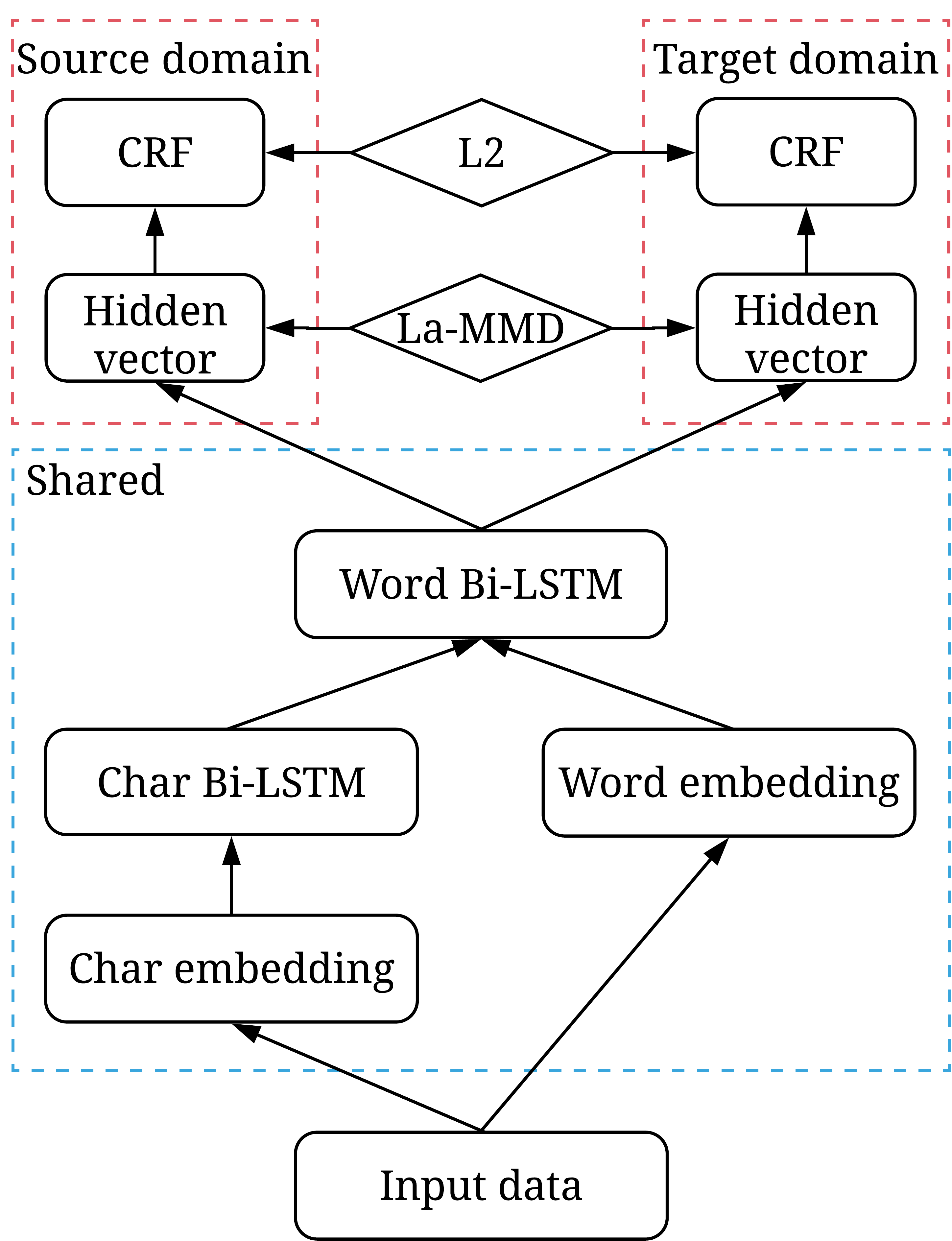}
			\caption{La-DTL framework for language like English.
			}
			\label{fig:eng_model}
		\end{figure}
		
		To show that La-DTL could be applied in transfer learning for NER scenario with mismatched named entity types and languages like English, we conduct this experiment transfer from CoNLL 2003 English NER to \textit{TwitterNER}. The four entity types in CoNLL 2003 English NER are LOC, PER, ORG, and MISC. The ten entity types in \textit{TwitterNER} are company, facility, geo-loc, movie, musicartist, other, person, product, sportsteam, and tvshow. 
		
		The Joint-training method \cite{yang2017transfer} separates the CRF layers for each domain to bypass the label mismatch problem. Since our La-DTL is label-aware, we match four pairs of named entities between two CoNLL 2003 English NER and \textit{TwitterNER}: LOC with geo-loc, PER with person, ORG with company and MISC with other to compute $\mL_\text{La-MMD}$ and $\mL_{p}$, and leave six named entities unmatched. Following \citet{yang2017transfer}, We leverage char-level Bi-LSTM to generate better word representations, concatenate it with pre-trained word embeddings and feed concatenated embeddings to the word-level Bi-LSTM. 
		The framework used for language like English is illustrated in Figure~\ref{fig:eng_model}.
		
		We also convert all characters to lowercase and use the same word embeddings provided by \citet{yang2017transfer}\footnote{\scriptsize \url{https://github.com/kimiyoung/transfer}}. Also, we concatenate the training set and the development set for both domains and sample the same 10\% from \textit{TwitterNER} as \cite{yang2017transfer} to be target domain training data. Since \citet{yang2017transfer} merge training and development set into training data, both \citet{yang2017transfer} and we report the best performance in the target domain test set.

	\end{document}